\documentclass{article}

    \PassOptionsToPackage{numbers}{natbib}
\usepackage{adjustbox}
\usepackage{subcaption}
\usepackage{wrapfig}

\usepackage[preprint]{neurips_2023}

\usepackage{soul}
\usepackage[utf8]{inputenc} % allow utf-8 input
\usepackage[T1]{fontenc}    % use 8-bit T1 fonts
\usepackage{hyperref}       % hyperlinks
\usepackage{url}            % simple URL typesetting
\usepackage{booktabs}       % professional-quality tables
\usepackage{amsfonts}       % blackboard math symbols
\usepackage{nicefrac}       % compact symbols for 1/2, etc.
\usepackage{microtype}      % microtypography
\usepackage{xcolor}         % colors
\usepackage{tikz}

\usepackage{amsmath}
\usepackage{amssymb}
\usepackage{mathtools}
\usepackage{amsthm}
\usepackage{caption}
\usepackage[normalem]{ulem}
\theoremstyle{plain}
\newtheorem{theorem}{Theorem}[section]

\theoremstyle{definition}

\title{2-D SSM: A General Spatial Layer\\ for Visual Transformers}

\author{%
  Ethan Baron\thanks{Equal Contribution. Order determined by coin flip} \\
  Department of Computer Science\\
  Tel Aviv University\\
  \texttt{barone@mail.tau.ac.il} \\
  \And
  Itamar Zimerman* \\
 Department of Computer Science\\
  Tel Aviv University\\
  \texttt{zimerman1@mail.tau.ac.il} \\
  \AND
  Lior Wolf \\
  Department of Computer Science\\
  Tel Aviv University\\
  \texttt{wolf@mail.tau.ac.il} \\
}

\begin{document}

\maketitle

\begin{abstract}
A central objective in computer vision is to design models with appropriate 2-D inductive bias. Desiderata for 2D inductive bias include two-dimensional position awareness, dynamic spatial locality, and translation and permutation invariance. %Traditionally, these properties are obtained by the layers of the model {\color{red}WHAT DO YOU MEAN BY LAYERS AND NOT BY ARCHITECTURE}, or, less optimally {\color{red}WHY LESS OPTIMAL}, by the backbone architecture (size of patches, type of positional encoding), on by the employed data augmentation. 
To address these goals, we leverage an expressive variation of the multidimensional State Space Model (SSM). %, to create a novel theory-grounded layer. 
Our approach introduces efficient parameterization, accelerated computation, and a suitable normalization scheme. Empirically, we observe that incorporating our layer at the beginning of each transformer block of Vision Transformers (ViT) significantly enhances performance for multiple ViT backbones and across datasets. The new layer is effective even with a negligible amount of additional parameters and inference time. Ablation studies and visualizations demonstrate that the layer has a strong 2-D inductive bias. For example, vision transformers equipped with our layer exhibit effective performance even without positional encoding. %Our code is attached as supplementary. %
\footnote[1]{The implementation of the method is available at \url{https://github.com/ethanbar11/ssm_2d}}
\end{abstract}

\section{Introduction}

%\footnotetext{For the purposes of reproducing our main results, please refer to the following anonymous Git repository at: \url{https://anonymous.4open.science/r/ssm_2d_submission-E3F2/README.md}.}% Upon acceptance, we will share our entire git repository.}
% {\color{lightgray}
% \begin{itemize}
%     \item 2-D bias in transmfromrs + positoinal awareness + natual images
%     \item vision transfromers

%     \item regularization of transformers
% \end{itemize}
% }

Incorporating image-specific inductive bias into computer vision networks could play a crucial role in their success, by shaping the hypothesis space in a way that fits image data and improves generalization. Common ingredients of image-specific inductive bias include two-dimensional neighborhood structure, locality, translation equivariance and invariance, and extraction of hierarchical features.  Traditionally, it was injected into the model through the backbone architecture. However, more recently, it has been modeled as part of the data. For example, two-dimensional neighborhood structures are typically expressed in one of two ways: (i) Vision Transformers~\cite{dosovitskiy2021an} use 1-D positional encoding~\cite{vaswani2017attention}, which is considered weak inductive bias. (ii) ConvNets employ 2-D kernels, which provide strong priors on the underlying image structure~\cite{ulyanov2018deep}. 

Most ConvNets employ relatively small filters in the convolution layers, and the balance between local and global features is handled by increasing the receptive field with depth. However, other kernel sizes can be beneficial. For example, ConvNeXt improved ResNet by $0.7\%$ on Imagenet, by only increasing its kernel size from $3\times 3$ to $7\times 7$~\cite{liu2022convnet}. More generally, using fixed-size filters limits the type of dependencies the layer can capture.

 % \sout{This work aims to introduce a novel layer that captures a versatile mixture of both local and global spatial features while relying heavily on a two-dimensional neighborhood structure. To accomplish this, we extend the recent advancements in SSM-based layers, known for capturing various type of dependencies in 1-D to 2-D. Like other SSM-based layers, our layer has a strong bias toward position awareness and good parameter efficiency and is well-grounded in control theory.} \textcolor{red}
The objective of this study is to develop a new layer that is adept at integrating both local and global spatial features, with a particular focus on a two-dimensional neighborhood structure. We accomplish this by building on recent developments in 1-D SSM-based layers, which are renowned for capturing various types of dependencies. By extending this 1-D concept to 2-D, our layer is deeply rooted in control theory, and much like its predecessors, maintains a strong bias towards position awareness and parameter efficiency. %It is also deeply rooted in control theory.

\noindent{\bf Our main contribution\enspace} is the 2D-SSM layer, which is a new spatial layer based on Roesser’s model for multidimensional state space~\cite{2dssm_R}. We show that simple design choices, such as diagonalization and normalization, can make the layer numerically stable and efficiently computed without recurrence using a 2-D convolution (left panel of Fig.~\ref{fig:mainFig}). Our layer has some unique properties, including: (i) A strong inductive bias towards two-dimensional neighborhood and locality, which stems from the multi-dimensional recurrent rule. As far as we know, this novel concept does not appear in other layers, (ii) The new layer can capture unrestricted controllable context. The SSM parameters of the layer can be focused on short or long, horizontal, vertical, or diagonal dependencies (middle panel of Fig.~\ref{fig:mainFig}). (iii) The layer is parameter-efficient and can express kernels of any length via 8 scalars. Visualization and ablation studies demonstrate these key aspects of our layers. Finally, the layer is well-grounded in control theory, and further theoretical analysis shows that it generalizes S4ND~\cite{nguyen2022s4nd} and proves its greater expressiveness. 

Empirically, we showed that our layer can be used as a general-purpose booster for vision transformers (the schematic architecture is illustrated in the right panel of Fig. \ref{fig:mainFig}), with negligible additional parameters and computation at inference. Furthermore, it appears that our 2D-SSM surpasses standard methods, such as incorporating positional encoding, in effectively integrating positional bias into Vision Transformers.

\section{Background and Notations}

\noindent{\bf Framing\enspace} Our research delves into two emerging research domains. The first domain focuses on the development of multi-axes global convolution techniques. Although 1-D (long) global convolution has shown promise in 1-D sequence modeling, leveraging methods such as SSM~\cite{dao2022hungry,gu2021s4,gu2021combining,gupta2022diagonal,gss} or other recent approaches~\cite{fu2023simple},\cite{poli2023hyena},\cite{li2022makes}, its applicability and effectiveness in modern computer vision tasks remain uncertain. Our work aims to explore and highlight the potential of these techniques in this domain, by extending them into 2-D.

The second domain investigates the synergistic combination of attention and SSM in 1-D modeling across various domains\cite{ma2022mega,saon2023diagonal,islam2022efficient,zuo2022efficient,dao2022hungry,gss}. For example, the SSM-based H3~\cite{dao2022hungry} outperforms GPT-Neo-2.7B~\cite{black2021gpt} (as well as other transformers of the same size) with only 2 attention layers. However, the question of whether these components are complementarity in 2-D modeling remains unanswered. We provide empirical evidence supporting the complementary nature of these components.

\begin{figure}[t]
\centering
\includegraphics[width=1\textwidth]{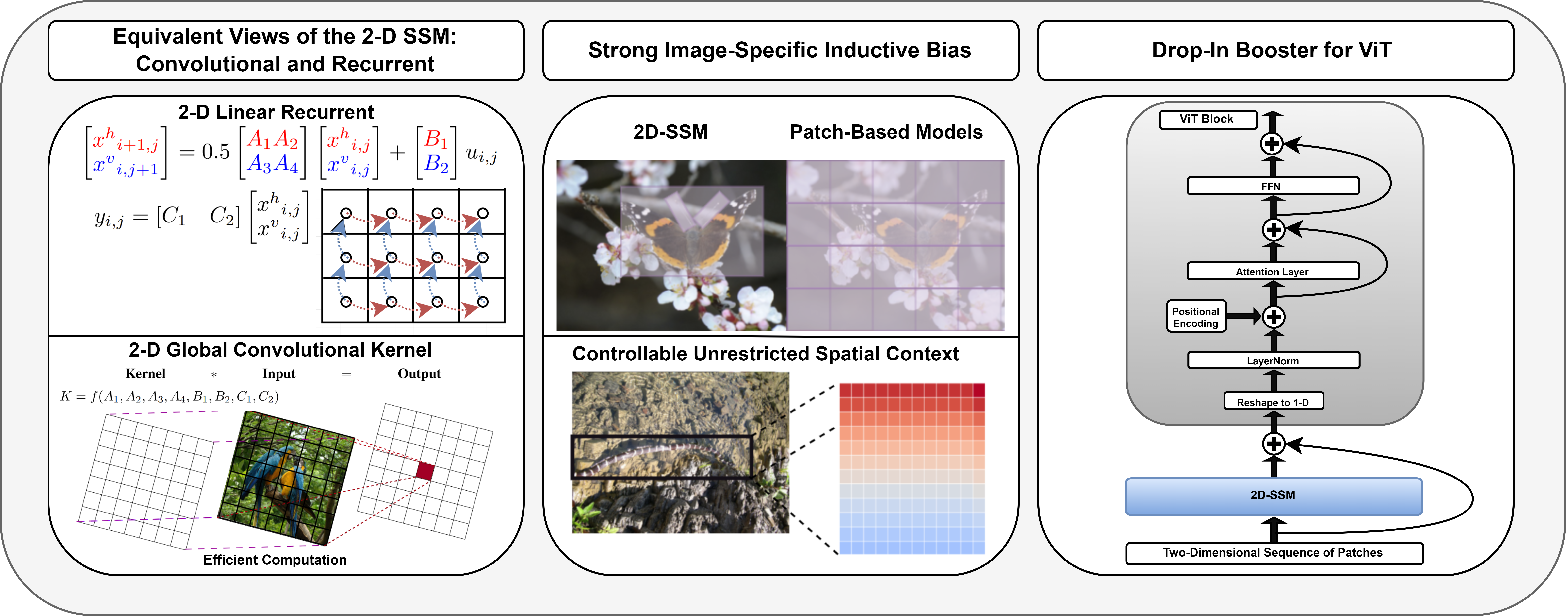}
\caption{(Left) The 2-D SSM layer is parameterized by A, B, C, and D.  It is built on top of a two-axis linear recurrent and can be efficiently computed using 2-D convolution. (Center) Since the layer is based on two-dimensional recurrence, it exhibits a strong bias toward positional awareness. The recurrent is unrestricted, allowing the layer to operate on 2-D sequences of any length. The values of $A_1, A_2, A_3$, and $A_4$ control the layer's focus, enabling it to capture short or long spatial dependencies in horizontal, vertical, or diagonal directions, as opposed to patch-based models. (Right) The layer can be easily integrated into ViT by applying it to the two-dimensional sequence of patches at the beginning of each transformer block.}
\label{fig:mainFig}
\end{figure}

\noindent{\bf State Space Model (SSM)\enspace}
%Before introducing state space layers, we provide background about the State Space Model (SSM). 
The state space model maps an input scalar function $u(t) : \mathbb{R} \rightarrow \mathbb{R}$ to a N-D latent state $x(t) \in \mathbb{R}^N$ before projecting it to an output signal $y(t) : \mathbb{R} \rightarrow \mathbb{R}$:
\begin{equation}
\label{eq:s4 equation}
%\begin{split}
    \dot{x}(t) = Ax(t) + B u(t) , \quad
    y(t) = Cx(t) + Du(t)
%\end{split}
\end{equation}
% where $x$ is the hidden variable, $A$ Tis the system matrix, $B,C$ are the input and output matrices, and $D$ is a parameter-based skip-connection.
The use of SSMs is widespread across numerous scientific disciplines and is closely associated with latent state models, such as Hidden Markov Models. There is a well-known connection between linear time-invariant SSMs, such as \ref{eq:s4 equation} and continuous convolution, thus allowing efficient training using the aforementioned equation as a discrete convolution. The S4 \cite{gu2021s4} and LSSL \cite{gu2021combining} layers leveraged the SSM as a black-box representation in deep sequence modeling, with learned parameters A, B, C, and D, and achieved strong results on several tasks, especially ones that require handling long-range dependencies, such as the Long Range Arena (LRA)~\cite{tay2020long}, audio generation~\cite{goel2022s}, and long-text processing ~\cite{gss,golub2016character,dao2022hungry}. 

%Recently, two insightful works by  analyzed 
The underlying reasons for the suitability of S4 and SSMs for modeling long sequences were recently analyzed~\cite{li2022makes,fu2023simple}. It was found that (i) employing global kernels with decaying structures, and (ii) using regularized kernels, are both critical design choices in this area. 
% parametrization with a sub-linear amount of parameters, are both critical design choices in this area. 

% In contrast to the original SSM, a recent variant called Diagonal Linear RNN (DLR)~\cite{gupta2022simplifying} was developed to provide the ability {\color{red}to dispose of this discretization step UNCLEAR THAT IS THIS STEP}. Let the system matrix be $A$ and the input-output matrices be $B$ and $C$. DLR has the following evolution:

% \begin{equation}
% \label{eq:distssm}
%     x_k = A x_{k-1} +  u_k  ,\quad y_k = C x_k\,,
% \end{equation}
% where $x_k$ is the hidden state at step $k$.

% {\bf Linear state space layers}
 % were first introduced by \citet{gu2021combining}, and further developed in the seminal S4 work of \citet{gu2021s4}. These layers realize the state-space model and have several unique properties, such as using complex parameters rather than real ones, computing the forward path with either a convolutional kernel or a recurrent rule, the ability to handle signals in both continuous and discrete spaces, and applying a global convolution along the axis of the signal. 

\noindent{\bf Roesser's 2D-State Space Model\enspace}
The attempt to extend the classical SSM for 2-D multi-axes systems was thoroughly studied in the past. \cite{2dssm_R,ssmdim1,ssmdim2,ssmdim3,ssmdim4,ssmdim5} notes a few different formulations of the problem.
We employ Roesser's SSM model~\cite{2dssm_R} as our discrete multi-axial model, which is the most general form of 2-axial and N-axial state-space models. As opposed to other SSMs already in use in machine learning, this SSM uses $M$ states, one per axis, rather than just one. The model in 2-D form is presented here:
\begin{equation} 
\label{eq:reqRule}
%\label{eq:outLayerSSM}
\centering
x_{i,j} = % 
\begin{bmatrix}
{x^h}_{i,j}\\
{x^v}_{i,j}
\end{bmatrix} , \quad 
y_{i,j} = % 
\begin{bmatrix}
{C}_{1}&{C}_{2}
\end{bmatrix} 
\begin{bmatrix}
{x^h}_{i,j}\\
{x^v}_{i,j} 
\end{bmatrix}  ,\quad% + D u_{i,j}
% \end{equation}
%
% \begin{equation}
% \centering
% \label{eq:reqRule}
\begin{bmatrix}
{x^h}_{i, j+1}  \\
{x^v}_{i+1, j} 
\end{bmatrix} = 
\begin{bmatrix}
A_1
A_2 \\
A_3
A_4
\end{bmatrix}
\begin{bmatrix}
{x^h}_{i, j}\\
{x^v}_{i, j}
\end{bmatrix}
+
\begin{bmatrix}
{B}_{1}\\
{B}_{2}
\end{bmatrix} u_{i, j}\,
\end{equation}
where the state $x_{i,j} \in R^{2N}$ is the concatenation of the horizontal ${x^v}_{i,j} \in R^N$ and vertical ${x^h}_{i,j} \in R^N$ states, the system matrices are $A_1, A_2, A_3, A_4 \in R^{N \times N}$ and the input and output matrices are  $B_1, B_2, C_1, C_2  \in R^N$.
There is also a learned parameter D that behaves as a skip-connection, omitted from now on for brevity.

\subsection{Notation}

The notation follows as closely as possible the notation used in the state-space layer literature~\cite{gu2021s4,gupta2022diagonal,gu2022parameterization}. Specifically, we use $H$ as the number of channels, $N$ as the state's hidden dimension, $L$ as the sequence length, $n_{ssm}$ as the number of non-shared channels, and treat $A,B,C,D \in \mathbb{R}$ as the system matrices. Note that for brevity the system matrices are treated as real-valued, even though we also test a complex version of them. The number of axes is set to $M$.

\noindent{\bf Signal dimensions\enspace} Although $N$-D SSM can be used in an $N$-Dimensional manner, since our paper focuses on using this model as a regularization method for ViT backbones and for simplifying the reading experience, we will treat it as a $2$-D SSM. 

$L_i \in \mathbb{R}$ is the sequence length along the $i$ axes,  $L_{tot} = L_1 * L_2$ is the total signal size, and  $L_{max} =\max({L_1,L_2})$.

{\bf Kernel Notation\quad} $K$ is the computed 2-D kernel such that $Y = U \ast K$, where $\ast$ denotes discrete convolution. 
\begin{align}
            y_{i,j} = C_1 {x^h}_{i,j} + C_2 {x^v}_{i,j} =       
    \sum_{0 \leq \hat{i} \leq i}\sum_{ 0 \leq \hat{j} \leq j} \big{(}C_1{k^h}_{\bar{i},\hat{j}} + C_2 {k^v}_{\bar{i},\hat{j}} ) \big{)} u_{\hat{i},\hat{j}} 
\end{align}

%{\color{red}
% Since the N-Dimensional SSM formulation is quite difficult to follow, we will simplify the formulation 

% From now onwards we will explain the case for 2-Dimensional SSM formulation.

% The $2$-D state-space model is the basis for our layer, so we use the same notations as S$4$. namely, we denote the Hidden state dimension as $N$, the sequence length at each axes as $L_1 , L_2$, and the .....

% In contrast,FOR x,y,z,tbd we used the same notations as mega. 
% \begin{itemize}
%     \item 
%     \item number of axes -$n$ (replace $M$)
%     \item $L_1 , L_2$ - Signal size dimensions, $L_{tot} = L_1 * L_2$ 
%     \item $L_{max} = \max({L_1,L_2}$)
% \end{itemize}
% sublinear
%}

\subsection{Other related work}

\noindent{\bf Multi-dimensional State Space Layers\enspace} 
As far as we know, S4ND~\cite{nguyen2022s4nd} is the only previous SSM-based layer that can naturally handle multidimensional data. S4ND is built on top of S4, and contains $M$ separate instances of S4, where $M$ is the number of axes. On the forward path, each S4 layer, which we denote as $SSM_g$ factors a one-dimensional local kernel $k_g$, and a global kernel $K$ is computed as the outer products of those local kernels.

\noindent{\bf Vision Transformer Backbones\enspace}
To demonstrate the versatility and efficacy of our 2-D SSM layer as a plug-and-play component compatible with various ViTs, we evaluate its performance when integrated into the following backbone architectures:
{\bf (i) ViT} The original ViT that employs self-attention on a 1-D sequence of patches; it used a learnable 1-D position encoding. 
{\bf (ii) Swin}
The Swin Transformer \cite{liu2021swin} refines ViT by incorporating hierarchical structure and local connections within windows. It employs a shifted windowing scheme and stage-wise processing to efficiently capture global context, which enhances its performance across various vision tasks.
{\bf (iii) Mega\enspace}
The Mega \cite{ma2022mega} model introduced a single-head gated attention mechanism enhanced with an exponential moving average, which imparts position-aware local dependencies to the attention mechanism. This approach addresses the limitations of the Transformer's attention mechanism, such as weak inductive bias. Mega demonstrates superior performance across a variety of tasks, including the Long Range Arena, neural machine translation, language modeling, and image and speech classification. %
{{\bf (iv) DeiT\enspace}
~\cite{touvron2021training} is an adaptation of ViT, incorporating a class-token designed for distillation purposes.}

\noindent{\bf Adding positional bias into transformers\enspace}
By design, Vision Transformers are permutation invariant, and thus a lot of work was put into injecting bias into them. Besides the standard positional encoding, the following methods are proposed:  

\noindent{\bf Exponential Moving Average (EMA)\enspace}
The EMA is a common technique for smoothing time-series data and prioritizing recent data points. It is computed using $EMA_t = (1 - \alpha) \cdot EMA_{t-1} + \alpha \cdot u_t$, where $EMA_t$ is the current EMA value, $EMA_{t-1}$ is the previous value and $\alpha$ is the smoothing factor. It is being used in MEGA \cite{ma2022mega} to incorporate positional awareness bias into the attention.

\noindent{\bf Other 2-D bias contributions\enspace}
By design, Vision Transformers are permutation invariant, and thus a lot of work was put into injecting 2-D bias into them. A particular research direction emphasizes the introduction of positional bias via various positional encoding methods. For instance, the Swin Transformer \cite{liu2021swin} employs a learnable bias term referred to as relative positional bias. In contrast, Convit \cite{d2021convit} utilizes the relative positional bias while modeling it as a soft inductive bias. This is achieved by initializing the attention function as a convolution and allowing the model to converge toward diverse attention forms. In alternative avenues of research, efforts have been made to modify the attention window through diverse techniques, such as incorporating a two-dimensional local bias by cropping \cite{dong2021cswin} the attention window or integrating convolutional neural networks (CNNs) with attention mechanisms \cite{dai2021coatnet}, \cite{li2022uniformer}. Recently, an intriguing outcome in image classification was accomplished when the EMA mechanism was applied to image patches, as described in \cite{ma2022mega}.
 % \sout{study aims to develop a layer that functions as a soft inductive bias, facilitating the insertion of a two-dimensional learnable low-parameter convolution mechanism.}

\section{Method}
In this section, we present the core of the $2$-D SSM layer, which is our main technical contribution. This layer maps a $2$-dimensional sequence $u_{i,j}$ to $y_{i,j}$, where  $u_{i,j},y_{i,j} \in \mathbb{R}$ for all $ 0 \leq i \leq L_1$ and $ 0 \leq j \leq L_2$. Similarly to previous SSM-based layers, we extend the core of our layer to a multi-directional and multidimensional layer, detailed in Appendix \ref{sec:modelExtension}.

%In subsections \ref{section: methodMain},\ref{Parameterization},\ref{subsection:Computation} we describe the layer in detail, and in subsections \ref{complexSSM},TBD several design choices are discussed. Sec. Furthermore, in  Sec. \ref{modelAnlysis} the theoretical benefits of our model are illustrated.
%\noindent{\bf The $2$-D SSM layer} 
%$begin{equation}% \forall i : 0 \leq i \leq L_1 , \forall j : 0 \leq j \leq L_2 :\quadend{equation}  
%, which is based on Eq. \ref{eq:reqRule}. Similar to previous work, we used a %, which we mainly use to incorporate two-dimensional position awareness bias into ViT.   
%
%
% The obtain layer proved as very effective and ...
% we leveraging the Multidimensional State Space Model, which is a generalization of EMA, and replace it with the EMA in the MEGA layer. it to add bias towards 2-D .  
% Vision transformers perform very well on several problems in CV altgought they have much less image-specific inductive bias than CNNs, they perform very well on several problems in CV. This observation open the door for a promising research direction: adding image-specific inductive bias to transformers. While several works try to improve the inductive bias of ViT, in this work we choose to focus on the mega transformer.   The original ViT use patches to add spatial bias to the transfromer,    Recently, a new type of transfromer, which reuglirzed via EMA achieved promising results on imagenet without any inductive bias towards images. 

\subsection{2-D Recurrent State Space Model as Convolution} 
\label{section: methodMain}
 Eq.~\ref{eq:reqRule} is defined in a recursive manner. However, for reasons of efficiency, it is best to define operators in closed form and in a way that is concurrent with parallel processing.  Inspired by previous work~\cite{gu2021s4,gupta2022diagonal,ma2022mega,gu2022parameterization}, we exploit the fact that the SSM is linear and can therefore be expressed as a convolution with a kernel $K$. To do so, we first unroll the recurrent rule and then describe it in a closed-form formulation.

For simplicity, we assume that the initial states are zeros  $\forall j \geq 0 :  {x^h}_{-1,j} = 0,$ and $ \forall i\geq 0:  {x^v}_{i,-1} = 0 $. The horizontal and vertical states at $i=0,j=0$ are:
\begin{equation}
{x^h}_{0,0} = B_1 u_{0,0}, \quad {x^v}_{0,0} = B_2 u_{0,0} 
\end{equation}

By applying the recurrent rule once at each axis:
\begin{align}
\label{eq:rec1}
    {x^h}_{1,0} = A_1 B_1 u_{0,0} + A_2 B_2 u_{0,0} + B_1 u_{1,0} ,\quad 
    {x^v}_{0,1} = A_3 B_1 u_{0,0} + A_4 B_2 u_{0,0} + B_2 u_{0,1}
\end{align}
\begin{align}
%\label{eq:rec1}
    {x^v}_{1,0} = B_2 u_{1,0}, \quad  {x^h}_{0,1} = B_1 u_{0,1} 
\end{align}

Next, we compute ${x^h}_{1,1}$ given ${x^h}_{0,1}$ and ${x^v}_{0,1}$.
\begin{align}
\label{eq:rec2}
 {x^h}_{1,1}=A_1 A_3 B_1 u_{0,0} + A_1 A_4 B_2 u_{0,0} +  A_1 B_2 u_{0,1}+ A_2 B_1 u_{0,1} + B_1 u_{1,1}\,,\\
 =k^h_{1,1}u_{0,0} + k^{h}_{1,0}u_{0,1} + k^{h}_{0,0}u_{1,1} 
\end{align}
and in general
\begin{align}
\label{eq:simpleConvH}
{x^h}_{i,j} &= \sum_{0 \leq \hat{i} \leq i}\sum_{ 0 \leq \hat{j} \leq j} {k^h}_{\bar{i},\hat{j}} u_{\hat{i},\hat{j}} \quad , \
% \end{align}
% \begin{align}
%\label{eq:simpleConvV}
{x^v}_{i,j} = \sum_{0 \leq \hat{i} \leq i}\sum_{ 0 \leq \hat{j} \leq j} {k^v}_{\bar{i},\hat{j}} u_{\hat{i},\hat{j}}\,, 
\end{align}
where as explained in the notation, each element ${k^h}_{\hat{i},\hat{j}}$ , ${k^v}_{\hat{i},\hat{j}}$ is an aggregation of $A_1,A_2,A_3,A_4,B_1,B_2$ multiplications (e.g Eq. \ref{eq:rec2}) and is associated with a single path from coordinate $(0,0)$ to $(i,j)$, as presented in Fig \ref{fig:paths}.

By plugging Eq.~\ref{eq:simpleConvH} in Eq.~\ref{eq:reqRule} one obtains:
\begin{align}
    y_{i,j} = C_1 {x^h}_{i,j} + C_2 {x^v}_{i,j} =
\sum_{0 \leq \hat{i} \leq i}\sum_{ 0 \leq \hat{j} \leq j} \big{(}C_1{k^h}_{\bar{i},\hat{j}} + C_2 {k^v}_{\bar{i},\hat{j}}  \big{)} u_{\hat{i},\hat{j}} 
\end{align}
and the convolutional kernel $K$ is 
\begin{equation} 
\label{eq:K_manifest}
\forall i,j : K_{i,j} = C_1{k^h}_{\bar{i},\hat{j}}+ C_2 {k^v}_{\bar{i},\hat{j}}
\end{equation}

% \begin{figure}[t]
% \centering
% \begin{tikzpicture}[scale=0.6]
% \begin{scope}
%     \draw (0,0) grid (5,5);
%     \draw[,-,dashed]    (0.4,0.4) -- (0.4,4.5);
%     \draw[->,dashed]    (0.4,4.5) -- (4.4,4.5);
%     %
%     \draw[-,dashed]     (0.4,0.4) -- (4.6,0.4);
%     \draw[->,dashed]    (4.6,0.4) -- (4.6,4.4);
    
%     \draw[-,dashed]     (0.6,0.4) -- (0.6,1.5);
%     \draw[-,dashed]     (0.6,1.5)-- (0.6,2.5);
%     \draw[-,dashed]     (0.6,2.5)--(3.6,2.5);
%     \draw[-,dashed]     (3.6,2.5)--(3.6,4.4);
%     \draw[->,dashed]     (3.6,4.4)--(4.3,4.4);

%     \draw[-,dashed]     (0.4,0.6) -- (1.5, 0.6);
%     \draw[-,dashed]     (1.5,0.6)-- (2.5,0.6);
%     \draw[-,dashed]     (2.5,0.6)--(2.5,1.4);
%     \draw[-,dashed]     (2.5,1.4)-- (4.4,1.4);
%     \draw[->,dashed]     (4.4,1.4)-- (4.4,4.4);
%     \filldraw [gray] (0.5,0.5) circle (3pt);
%     \filldraw [gray] (4.5,4.5) circle (3pt);
%     % \draw[thick,-]     (3.6,4.4)--(4.3,4.4);
%     \node[] (a1) at (5,-1){} ;
% \end{scope}
% \end{tikzpicture}
% \vspace{-8mm}
%  \caption{Examples of paths from coordinate $(\hat{i},\hat{j})=(0,0)$ to $(i,j)=(4,4)$. Each path represents a sequence of recursive calls for Eq. \ref{eq:reqRule}.}
%  \label{fig:paths}
% \end{figure}

\begin{figure}
  \centering
  \begin{minipage}{0.30\linewidth}
      \centering
          \begin{tikzpicture}[scale=0.7]
\begin{scope}
    \draw (0,0) grid (5,5);
    \draw[,-,dashed]    (0.4,0.4) -- (0.4,4.5);
    \draw[->,dashed]    (0.4,4.5) -- (4.4,4.5);
    \draw[-,dashed]     (0.4,0.4) -- (4.6,0.4);
    \draw[->,dashed]    (4.6,0.4) -- (4.6,4.4);
    
    \draw[-,dashed]     (0.6,0.4) -- (0.6,1.5);
    \draw[-,dashed]     (0.6,1.5)-- (0.6,2.5);
    \draw[-,dashed]     (0.6,2.5)--(3.6,2.5);
    \draw[-,dashed]     (3.6,2.5)--(3.6,4.4);
    \draw[->,dashed]     (3.6,4.4)--(4.3,4.4);

    \draw[-,dashed]     (0.4,0.6) -- (1.5, 0.6);
    \draw[-,dashed]     (1.5,0.6)-- (2.5,0.6);
    \draw[-,dashed]     (2.5,0.6)--(2.5,1.4);
    \draw[-,dashed]     (2.5,1.4)-- (4.4,1.4);
    \draw[->,dashed]     (4.4,1.4)-- (4.4,4.4);
    \filldraw [gray] (0.5,0.5) circle (3pt);
    \filldraw [gray] (4.5,4.5) circle (3pt);
    % \draw[thick,-]     (3.6,4.4)--(4.3,4.4);
    \node[] (a1) at (5,-1){} ;
\end{scope}
\end{tikzpicture}
      \caption{Examples of paths from coordinate $(\hat{i},\hat{j})=(0,0)$ to $(i,j)=(4,4)$. Each path represents a sequence of recursive calls for Eq. \ref{eq:reqRule}.\label{fig:paths}}
    \end{minipage}
%  \begin{adjustbox}{valign=t}
 %   \begin{subfigure}[t]{0.67\linewidth}
 \hfill
 \begin{minipage}{.65\linewidth}
      \centering
      \begin{tabular}{@{}c@{}c@{}c@{}}
        \includegraphics[width=1\linewidth,height=0.07\textheight]{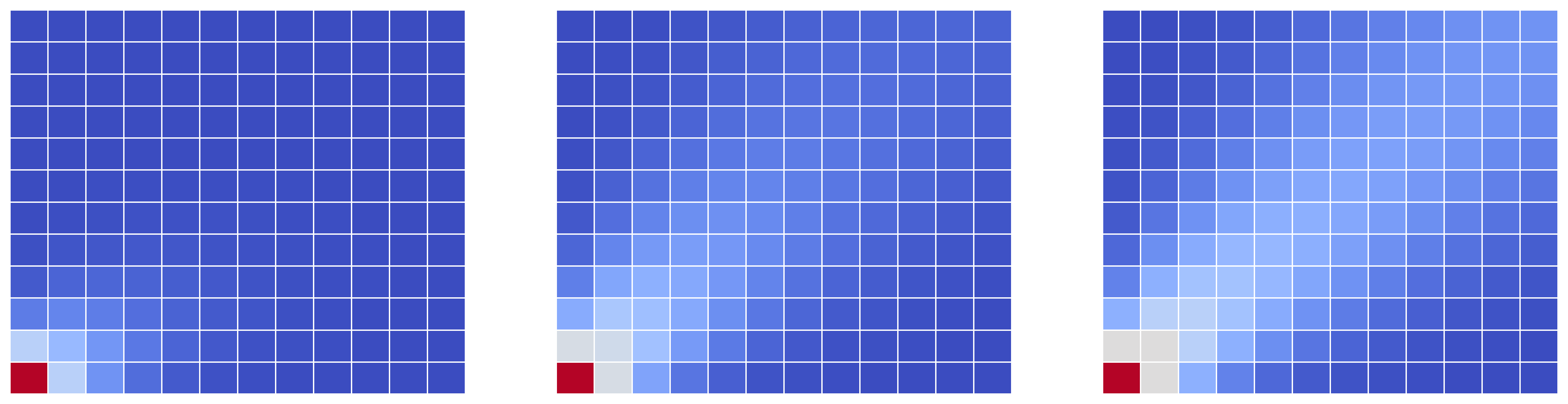} \\
        \includegraphics[width=1\linewidth,height=0.07\textheight]{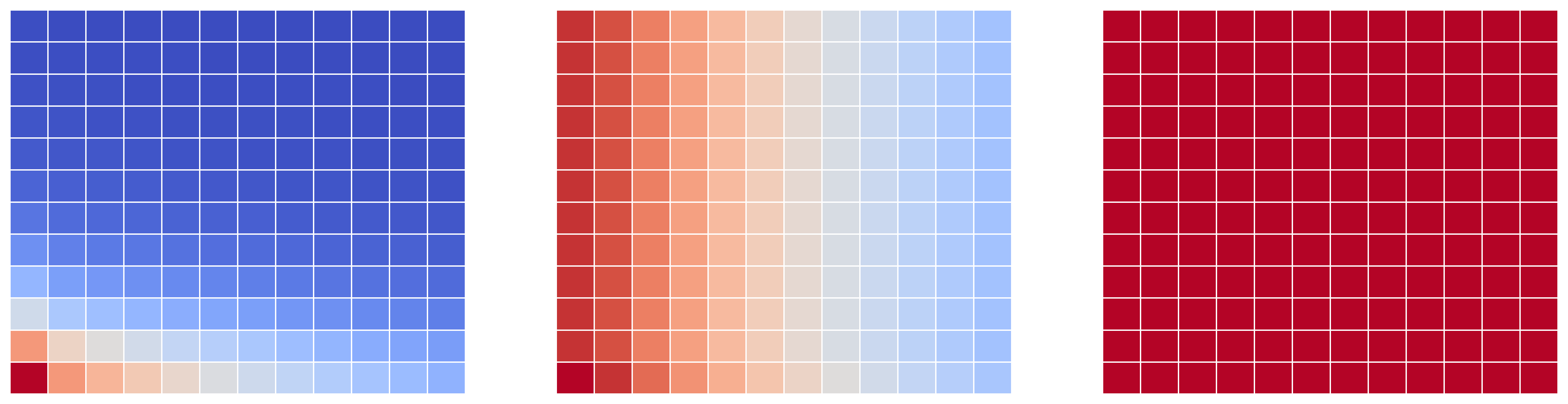} 
      \end{tabular}
      \caption{The kernels before and after the modifications of Sec.~\ref{par:relaxation}. Each column is created by the same $A_1...A_4,B_1, B_2, C_1, C_2 \in \mathbb{R}$ parameters. The first row is the normalized 2-D SSM formulation explained in \ref{eq:reqRule}, the second is the outcome of Eq.~\ref{eq:reqRule_normalized} and performing Eq.~\ref{eq:multi_by_2}, which is the kernel formulation we use. The bottom left corner of each heatmap is $K_{0,0}$. The figures demonstrate that before the relaxation, the kernels displayed a diagonal tendency while afterward, they exhibited a more diverse and versatile pattern.
      \label{Fig:kernalsCorrectiom}}
    \end{minipage}
%  \end{adjustbox}
  \hfill%
%  \begin{adjustbox}{valign=t}
 % \end{adjustbox}
  % \caption{{\color{red}TBD}}
  \label{fig:both}
\end{figure}

\subsection{Efficient and Stable Parameterization}
\label{Parameterization}
\noindent{\bf Parameter diagonalization\enspace}
Computing $K$ is difficult for two reasons. Firstly, the number of elements in each $k_{i,j}$ is exponential in $i$ and $j$, since it is equivalent to the number of paths from $(0,0)$ to $(i,j)$. Secondly, calculating the powers of non-diagonal matrices $A_1, A_2, A_3, A_4$ becomes costly when $L_1, L_2$ are large. To overcome these challenges, we parameterized the system matrices $A_1, A_2, A_3, A_4$ as diagonal. This allows for efficient summation of at most $2L_{max}$ elements in $k^{h}_{i,j}$. Although diagonalization limits the expressiveness of our SSM, previous works have shown its effectiveness in one-dimensional cases~\cite{gu2022parameterization,gupta2022diagonal,gss,gupta2022simplifying}.

\noindent{\bf Limiting the parameters\enspace}\label{limiting_paragraph} $A_i \in \mathbb{R}^{N \times N}$ is now a diagonal matrix. We'll denote the eigenvalues of $A_i$ by  $(\alpha_1,\alpha_2...\alpha_N)  := \mathbf{diag}(A_i)$. Each $\alpha_i$ value behaves separately until the computation of K (Eq. \ref{eq:K_manifest}). Thus, for $\alpha_i > 1 $, $ \lim_{z\to\infty} \alpha^{z} = \infty$. % Therefore, we limit $\alpha_i \in [0,1]$ by - instead of optimizing on $\alpha_i$ directly - optimizing on $\hat{\alpha_i} \in \mathbb{R}$ and calculating $\alpha_i = \textit{sigmoid}(\hat{\alpha_i})$.%
Therefore, we limit $\alpha_i \in [0,1]$ by parameterized it by $\alpha_i := \textit{sigmoid}(\hat{\alpha_i})$ and optimizing $\hat{\alpha_i}$ instead.

\noindent{\bf Normalization\enspace} There is a scaling problem that arises from Eq. \ref{eq:rec2}. Since $k^{h}_{1,1} = A_1A_3B_1+A_1A_4B_2$, even if $ A_1A_3B_1, A_1A_4B_2  \leq 1$ ,$k^{h}_{1,1}$ can be greater than 1. The same behavior makes the kernel explode.

We would like to keep $\forall{i,j} :0 \leq k^h_{i,j} \leq 1$ and thus we employ a straightforward normalization mechanism, in which we divide by two each time we compute Eq. \ref{eq:reqRule}. This Eq. is thus replaced by:
% \begin{equation}
% \centering
% \label{eq:reqRule_normalized}
% \begin{bmatrix}
% {\color{red}{x^h}_{i+1, j}}  \\
% {\color{blue}{x^v}_{i, j+1}}
% \end{bmatrix} = 
% 0.5
% \begin{bmatrix}
% {\color{red}A_1}
% {\color{red}A_2} \\
% {\color{blue}A_3}
% {\color{blue}A_4}
% \end{bmatrix}
% \begin{bmatrix}
% {\color{red}{x^h}_{i, j}}\\
% {\color{blue}{x^v}_{i, j}}
% \end{bmatrix}
% +
% \begin{bmatrix}
% {\color{red}{B}_{1}}\\
% {\color{blue}{B}_{2}}
% \end{bmatrix} u_{i, j}
% \end{equation}

\begin{equation}
\centering
\label{eq:reqRule_normalized}
\begin{bmatrix}
{x^h}_{i+1, j}  \\
{x^v}_{i, j+1} 
\end{bmatrix} = 
0.5
\begin{bmatrix}
A_1
A_2 \\
A_3
A_4
\end{bmatrix}
\begin{bmatrix}
{x^h}_{i, j}\\
{x^v}_{i, j}
\end{bmatrix}
+
\begin{bmatrix}
{B}_{1}\\
{B}_{2}
\end{bmatrix} u_{i, j}
\end{equation}

\noindent{\bf Relaxation of the kernel\quad}\label{par:relaxation} When imposing a "divide by two" normalization for every $k_{i,j}$, we impose a kernel formulation that is much more biased towards modeling diagonal kernels, i.e., if $|i-j|$ is small, $k^h_{i,j}$ is much larger than when $|i-j|$ is large.
%{\color{red}REPHRASE (not to be confused with our parameters which are diagonal matrices reference the structure of the kernel itself)}
%rather than spatial kernels, which are also 
%{\color{red}NEITHER? horizontal nor vertical}. 

% This problem arises from the fact $\forall{i}$,    $ k_{i,0}$ is receiving flow only from $k_{\hat{i},0}$,   $\hat{i} <i$
% and thus there is only one route from ${\hat{i},0}$ to $i,j$ and $ \forall{\alpha_{i,0}}, \alpha_{i,0} =1$. 

Thus, we relax the normalization in the first row and column as follows. When calculating $k^h_{i,0} ,k^h_{0,j}$ we use Eq. \ref{eq:reqRule}. Additionally, for $K_{i,0}$, $K_{0,j}$, we use $\hat{C_1} = 2C_1, \hat{C_2}=2C_2$ in the following manner:
\begin{equation}
\label{eq:multi_by_2}
    K_{0,j} = \hat{C_1}k^h_{0,j} + \hat{C_2}k^v_{0,j} 
\end{equation}

{Figure \ref{Fig:kernalsCorrectiom} illustrates examples of different kernels before and after the relaxation.}
% \sout{Fig.\ref{Fig:kernalsCorrectiom} shows examples of kernels with different parameter values $A_1...A_4, B_1, B_2,C_1, C_2$ (fixed in each column) and the kernel one obtains by each one of the formulations. The first row of Fig.~\ref{Fig:kernalsCorrectiom} depicts the unmodified kernel, which is heavily biased toward the diagonal, and the horizontal and vertical axes are much lower. In the second row, we can see that when using Eq. \ref{eq:reqRule} for $k^h_{i,0} ,k^h_{0,j}$, the kernels change dramatically and are able to capture long horizontal and vertical dependencies. 

% The third row shows the final kernel we use, after applying Eq. \ref{eq:multi_by_2}, which again fixes an imbalance in our model against the first and last column.}% This operation is employed because of the fact that in contrast to the other parts of the kernel, we defined $x^h_{-1,j} = x^v_{-1,j} = x^h_{i,-1} = x^v_{i,-1}=0$ and thus their state value is lower.

%The effect of the last two operations on the kernels, which enhance the model performance significantly, is shown in Fig. \ref{Fig:kernalsCorrectiom}.

We note that these modifications are straightforward and probably not optimal. The development of optimal normalization according to the kernel formulation is an interesting direction for future work.

\subsection{Computation and Complexity} 
\label{subsection:Computation}
\noindent{\bf Training Complexity\enspace} The training process has two parts. First, we calculate the kernel K. The calculation itself is explained thoroughly in Appendix \ref{appendix:complexity} and has a time complexity of $O(L_{tot}L_{max}N)$, which is not dependent on B, and it is much faster than naive computation thanks to several design choices, such as parameters diagonalization, pre-processing, and a sophisticated caching procedure.

Next, we apply a convolution between $K$ and $U$, using the classical procedure of FFT, element-wise multiplication, and inverse FFT. The complexity of this step is $O(BL_{tot}\log(L_{tot}))$ where $B$ is the batch size. Hence, the total complexity is:
\begin{equation}
O(L_{max}NL +B\log(L_{tot})L_{tot})
\end{equation}

\noindent{\bf Inference Complexity\enspace}
During inference, the 2-D SSM can pre-compute the convolution kernels, resulting in an additional computation cost of only the 2-dimensional convolution of the layer input and the signal, which takes $L_{tot} \log (L_{tot})$ operations. Therefore, the additional computation overhead relative to the vanilla ViT is minimal, as the quadratic complexity dominates the overall complexity.

\label{subsection:complexSSM}
\subsection{Complex and Real SSMs}
% \label{complexSSM}
While most SSM-based deep learning models, e.g S4 \cite{gu2021s4} or DLR~\cite{gupta2022simplifying}, use a complex SSM, MEGA used EMA, which can be interpreted as a restriction of the diagonal SSM to real numbers. Our 2-D SSM layer can be built over real or complex parametrization. %Empirically, we found that the real 2D-SSM outperformed the complex 2D-SSM, see Sec.~\ref{sec:experiments}.%
The complex-SSM based model that we examined is described in detail in Appendix \ref{subsec:complexModelDeatils}. %To achieve fair comparison we applied the same normalization techniques that was used in \ref{subsubsection:stableMormalizatoinAndLimiting} on the real and imaginary parts of the complex parameters.}
% for a hypothesis. %for which several explanations and justifications are presented in this subsection.
 
%\label{subsection:complexRealSSM}
%Further investigation regarding real vs. complex SSM is presented in Sec.~\ref{sec:modelAnlysis} and Sec.~\ref{sec:experiments}. 

\section{Model Analysis}
\label{sec:modelAnlysis}

In this section, we study the expressiveness of the 2-D SSM layer, as well as its unique spatial inductive bias.

\subsection{Expressiveness}
We compare the expressiveness of our 2-D SSM layer with S4ND \cite{nguyen2022s4nd}, a very recent layer that is also based on multidimensional multi-axes SSM. We first introduce the key differences between our layer and S4ND, and then demonstrate the expressiveness gap. %We then prove that our method can express two-dimensional full-rank kernels, as opposed to S4ND, which can only express rank-1 kernels. %The bottom line is that though our model has more inductive bias towards two-dimensional locality, it spuriously does not detract from the expressiveness power of the model, at least in terms of tensor rank.

\noindent{\bf The relationship between 2-D SSM and S4ND\enspace} The main difference between S$4$ND and $2$-D SSM is that S$4$ND runs a standard 1-D SSM over each axis independently, %that is, the model learns a different set of local functions per axis, 
and those functions are combined to form a global kernel. In contrast, our model learns multi-dimensional functions over multi-axes data directly. This difference arises from the fact that the $2$-D SSM has additional system matrices, $A_2, A_3$, which aggregate and process information from different axes. 

\noindent{\bf 2D-SSM is a generalization of S4ND\enspace} When restricted to 2-dimensional space, given a 2-D SSM model, S4ND is obtained by restricting $A_2, A_3$ to be zeros, and setting $A_1, A_4$ as the system matrices of S4ND. Additionally, to replicate S4ND in 2D-SSM, one should initialize the states % first rows and first columns
of the 2D-SSM by the kernels factorized from S4ND.
% It should be noted that the parameters of S4ND are complex, while our parameters are real or complex, as described in~\ref{subsection:complexRealSSM}.

\noindent{\bf Tensor rank as a criterion for expressiveness\enspace}
Over the years, several criteria were proposed for measuring the expressiveness of neural and classical models, such as VC-dimension~\cite{vcdim}, norms, and Rademacher complexity~\cite{bartlett2002rademacher}. Inspired by ~\cite{cohen2016expressive}, we employ tensor rank as our measure, and prove the followings theorems:

\begin{theorem}
\label{theorem:2dssm-full-kernals-dec}
  The $8$ parameters of the $2$-D SSM can express full-rank kernels 
\end{theorem}

\begin{theorem}
\label{theorem:s4nd-1d-kernals-dec}
  S$4$ND can express kernels of rank $1$ only.
\end{theorem}

\noindent{\bf Assumptions\enspace} For simplicity, we assume that both the $2$-D SSM and S$4$ND layers contain one channel and one hidden dimension. In this case, the SSM matrices are scalars.
 % complex scalars in S$4$ND and real or complex scalars in our method, but the result holds even for real SSM matrices. 
%
When the assumption about the number of channels is omitted,
the rank of kernels that S4ND can express increases from $1$ to $N$. However, this comes with the cost of $MNr$ additional parameters, where $M$ is the number of axes and $r$ is the required rank for S4ND. It should be noted that the authors of S4ND did not evaluate the performance of models with $r > 1$.

Under these assumptions, the proof of Theorem \ref{theorem:2dssm-full-kernals-dec} is specified in Appendix \ref{theorem:2dssm-full-kernals}.
The proof of \ref{theorem:s4nd-1d-kernals-dec} is 
trivial, and derives from the fact that to compute a global multi-axis kernel $K$, S4ND takes the outer product operation on the per-axis kernels $k_m \in \mathbf{C}^{L_m\times 1}$ for all $m \in [M] $. Since each kernel is a vector, it is clear that: %
% it is clear from standard rank considerations that:
\begin{align}
    \textbf{rank}(K) = \textbf{rank}(k_1 \otimes k_2 \otimes \hdots k_M) = 1
\end{align}

\subsection{Image-Specific Inductive Bias}
%This subsection explains which aspects of inductive bias are expressed in our layer.
%During the early stages of this project, our north star question was how to incorporate image-specific inductive bias into ViTs and CNNs. This subsection explains which aspects of inductive bias are expressed in our method, and which considerations and design choices reflect those aspects.

\noindent{\bf Two-dimensional Position-Awareness\enspace} Our layer is grounded by a two-dimensional linear recurrent (Eq. \ref{eq:reqRule} ,Fig. \ref{fig:mainFig},left), as a result, positional information is taken into account by design when the kernel K is computed from the parameters $A, B, C,$ and $D$. This is a unique property without a counterpart in other modern layers. For instance, transformers lack positional bias and rely on additional positional encoding, while convolutional layers do not inherently encode explicit positional information; however, they learn to extract features based on local spatial dependencies.

Furthermore, as can be seen in Sec.~\ref{paragraph:runwithoutPE}, our method is highly effective in inserting positional bias into the transformer, even outperforming positional encoding in some cases. %

% While the performance of MEGA on images without positional encoding collapses, the same method with our 2DSSM layer is relatively robust to the complete removal of the positional encoding. 
 %
\noindent{\bf Controllable Unrestricted Spatial Context\enspace} 
A significant limitation of both CNNs and transformers is the need to choose an appropriate patch size, which can result in a loss of global or local context. In contrast, our layer implements a controllable global convolution, which benefits from a global context, and can control the effective receptive field and efficiently capture spatial local dependencies. Moreover, our layer is not confined to patch-like dependencies and can effectively capture diverse local and global features in horizontal, vertical, or diagonal directions (See \ref{fig:mainFig}, middle).
\noindent{\bf Modeling Symmetries\enspace} 
Similar to other CNNs, our layer exhibits translation equivariance as the kernels slide across the input during computation. Additionally, as detailed in Appendix \ref{sec:modelExtension}, our layer's core extends over multiple directions, which allows it to accommodate rotations and reflections naturally.
%
% \noindent{\bf Decaying Kernels} 
% \citet{li2022makes} discuss the importance of decaying kernels in the context of positional awareness and long-range dependencies in 1-D. While the meaning of decaying kernels is clear for one axis, for two axis the situation is more complex, since there are several possible decay orders. For reasons of symmetry, we focus on the following order of decay:
% \begin{align}
%     \forall i,j,\hat{i},\hat{j} : \quad i < \hat{i} ,  j < \hat{j} \rightarrow \quad K_{i,j} > K_{\hat{i},\hat{j}}\,.
% \end{align}
% %an order which $k$ . 
% When considering Eq.~\ref{eq:reqRule}, and the parametrization of the state matrices in Sec.~\ref{subsubsection:stableMormalizatoinAndLimiting}, which limits their value to be between $0-1$, as well as the normalization in Eq.~\ref{eq:rec2}, it is clear that there is an exponentially decaying trend over the diagonals of $K$.

\noindent{\bf Parameter Efficiency\enspace}
In contrast to CNNs, which parameterize filters of size $H \times W$ with at least $HW$ parameters, our layer has a fixed and small number of parameters ($9$ parameters per channel, $A_1,A_2,A_3,A_4,B_1,B_2,C_1,C_2,D \in  \mathbb{R}$), however, those parameters can be expanded into unbounded two-dimensional kernels. Furthermore, we use parameter sharing for the SSM parameterization across channels, similarly to CNNs or state-space layers and donate $n_{ssm}$ as the number of non-shared channels.

% In contrast to CNNs, which parameterize filters of size $H \times W$ with at least $HW$ parameters, it is important to note that even though our parameterization has a fixed and small number of parameters ($8$ parameters per channel), it can be factorized using unbounded two-dimensional kernels. Furthermore, we use parameter sharing for the SSM parameterization across channels, similarly to CNNs or state-space layers and donate $n_{ssm}$ as the number of non-shared channels.

% {\color{red}
% \noindent{\bf Features Handling\quad}
% Using patches is a well-known practice for incorporating image-specific inductive bias. These patches restrict the types of dependencies that can be captured by a single layer, and their size becomes a sensitive hyperparameter that defines a trade-off between local and global dependencies. 
% The dependencies that can be captured by
% 2. Flexible Kernels - behind patches

% 1. Hierarchical features
% }

\section{Experiments}
\label{sec:experiments}
%In this section, we evaluate our method empirically. We first detail our experimental setup in \ref{sec:setup}. Then, in Sec.\ref{sec:imageClassification} we present our main findings, which show that using 2D-SSM can boost the performance of strong models, such as MEGA and ViT. We conclude the empirical evaluation by providing additional insights via ablation studies and meaningful visualizations. 
% We start by details our experimental setup in XXX, then present results on image-classification task 
%In all benchmarks, we use the parameters of the backbone architecture as is, without trying to adjust it.

% \subsection{image classification}
%\input{cifar100-tinyimagenet}

%\noindent{\bf Experimental setup\quad} 
%\label{sec:setup}
We assess our 2-D SSM layer as an inductive bias within various ViT-based backbones. We demonstrate the universality of our method by incorporating it into various backbones, such as ViT, DeiT, Swin, and report improved results over the baselines on ImageNet-1k, Celeb-A, Tiny-ImageNet and CIFAR-100, with negligible additional parameters and without hyperparameter tuning, except for stochastic depth. {For a comprehensive overview of the experimental setup, please see Appendix \ref{sec:Experimentalsetup}}

\begin{table}[b]
\begin{minipage}[t]{0.62\linewidth}
\caption{Results using ViT, MEGA and Swin backbones on the Tiny ImageNet (T-IN) and  CIFAR-100 (C100) datasets. No hyperparameter tuning except for stochastic depth.}
\label{tab:Small Dataset}
\small
\centering
\begin{tabular}{@{}l@{~}c@{~}c@{~}c@{~}c@{}}
\toprule
Model & C100 & T-IN & Train Time  & \# Of Parameters  \\
\midrule
ViT             & 73.81 &	57.07 & 1x    & 	2.71M (1x)       \\     
ViT w/ S4ND.     & 72.60 &	56.10 &  2.22x   & 	2.72M (1.003x)       \\ 
ViT w/ SSM-r.     & \textbf{74.07} &	57.56 &  2.66x   & 	2.72M (1.003)      \\ 
ViT w/ SSM-c.     & 73.91 &	\textbf{57.66} &  2.89x   & 	2.73M (1.01x)       \\ 
\midrule
Mega-ablate     & 74.82 &	56.43 & 1.1x  & 	2.75M (1x)   \\    
Mega        & 72.27 &	54.49 & 1.28x & 	2.98M (1.08xx)     \\   
Mega w/ S4ND    & 74.9 &	56.65 & 1.46x & 	2.80M (1.02x)    \\    
Mega 2-D SSM-r    & \textbf{76.02} &	\textbf{57.95} & 2.03x & 	2.80M (1.02x)    \\   
Mega 2-D SSM-c    & 75.09 &	56.51 & 2.03x & 	2.84M (1.03x)    \\   
\midrule
Swin             &  76.87 &  60.87  & 1x         & 	7.15M (1x)      \\
Swin-reprod.     & 	77.98 &  61.29  &  1x        &  7.15M (1x)      \\
Swin w/ EMA      &	77.01 &  60.13  & 1.39x      & 	7.52M (1.05x)   \\
Swin w/ S4ND      & 	79.26 &  64.6  & 1.29x      & 	7.18M (1.004x)    \\
Swin w/ SSM-r      & 	\textbf{80.12} &  \textbf{65.77}  & 2.16x      & 	7.25M (1.01x)   \\
Swin w/ SSM-c      & 	3.28 &  12.76  & 2.18x      & 	7.26M (1.02x)   \\
% {\color{red}MEGA-ablate} &  	66.97 & {\color{red}TBD} & 5.91M \\
% MEGA &  69.73 & 89.76 & 6.21M \\
% \textsc{MEGA 2D-SSM (ours)} & \textbf{70.11}  % \textcolor{ForestGreen}{+0.38} 
% & \textbf{90.186} & 5.96M\\% \textcolor{ForestGreen}{+0.43} \\
\bottomrule
\end{tabular}
%\end{table} 
\end{minipage}
\hfill
\begin{minipage}[t]{0.35\linewidth}
%\begin{table}%{r}{0.5\linewidth}
\caption{Results for DeiT and Swin of different sizes on the Celeb-A dataset.}
\label{tab:CelebA}
\small
\centering
\begin{tabular}{lccc}
\toprule
Model & Top 1  &  \#Param\\
\midrule
DeiT-T & 88.43 & 5.532M \\
DeiT-T w. SSM-r & 89.76 & 5.537M \\
DeiT-T w. SSM-c & \textbf{89.84} & 5.541M \\
\midrule
DeiT-S & 89.66 & 21.681M \\
DeiT-S w. SSM-r & 90.24 & 21.688M \\
DeiT-S w. SSM-c & \textbf{90.38} & 21.691M \\
\midrule
DeiT-B & 90.13 & 85.829M \\
DeiT-B w. SSM-r & 90.45 & 85.841M \\
DeiT-B w. SSM-c & \textbf{90.73} & 85.845M \\
\midrule
Swin-T & 91.48 & 27.550M \\
Swin-T w. SSM-r & 91.68 & 27.556M \\
Swin-T w. SSM-c & \textbf{91.78} & 27.558M\\
\bottomrule
\end{tabular}
\end{minipage}
\end{table}

\noindent{\bf Swin, DeiT and ViT\enspace} We adopted a straightforward approach to incorporate our 2-D SSM layer into the aforementioned backbone structures. Prior to the Transformer Block, we apply the 2-D SSM to the input signal $u \in R^{L_1 \times L_2 \times D}$, as illustrated in Fig.~\ref{fig:mainFig}.
We highlight that in the case of Swin Transformer, the 2-D SSM operates on the patch level and not the window level and therefore injects additional 2-D bias between windows.

We tested \textbf{Swin and ViT} over the small datasets Tiny-ImageNet and CIFAR-100, using the results reported by \cite{lee2021vision} as baseline. As shown in Tab. \ref{tab:Small Dataset}, with the ViT backbone, we improve $0.8\%$ on CIFAR-100 and $0.5\%$ on Tiny-ImageNet, and with the Swin backbone, we improve $3.25\%$ on CIFAR-100 and $~4.25\%$ on Tiny ImageNet. %{\color{red}MEGA?}%Our runs clearly demonstrates the superiority of our layer in terms of accuracy when compared with baseline methods or alternative approaches. 

The \textbf{DeiT {and Swin}} backbones were tested on the large-scale Celeb-A dataset~\cite{liu2015faceattributes}. This dataset involves a 40-way multi-label attribute classification. We report aggregate accuracy across all 40 tasks. As can be seen in Tab.~\ref{tab:CelebA}, the complex version outperforms the real in all experiments and achieves $1.41\%$, $0.72\%$, $0.6\% $, and $0.3\%$  improvements over the baseline for DeiT sizes Tiny, Small, and Base, and Swin-Tiny respectively.
\begin{figure}[htbp]
  \captionsetup{type=table} % Set caption type to "table"

  \begin{minipage}[t]{0.36\linewidth}
  \small
    \centering
    \small
\caption{Accuracy on the CIFAR-10 grayscale classification task, which is part of the Long Range Arena.}

\label{tab:lra}
\centering
\begin{tabular}{lc}
\toprule
\textbf{Models} & \textbf{Image - LRA} \\
\midrule
Transformer~\cite{vaswani2017attention} & 42.94 \\
% Reformer~\cite{kitaevreformer}   &   68.50 \\
% Linformer~\cite{wang2020linformer}  &  38.56  \\
% BigBird~\cite{zaheer2020big}    & 40.83 \\
% Performer~\cite{choromanski2020rethinking} & 42.77 \\
% Luna-$256$~\cite{ma2021luna} & 47.86 \\
\midrule
S4-v1 &  87.26 \\
S4-v2  & 88.65 \\
CCNN-1D~\cite{knigge2023modelling} & 88.90 \\
CCNN-2D~\cite{knigge2023modelling} & 91.12 \\
S4ND~\cite{nguyen2022s4nd} & 89.90 \\
Hyena~\cite{poli2023hyena} & 91.20 \\
Mega Ablate &  81.00 \\
Mega  &  90.44 \\
\textsc{Mega 2-D SSM} & \textbf{91.31}\\ %\textcolor{ForestGreen}{(+0.87\%)} \\
\bottomrule
\end{tabular}
  \end{minipage}\hfill
  \begin{minipage}[t]{0.58\linewidth}
    \small
    \centering
    \caption{ImageNet-1K accuracy of MEGA variants.} 
\label{tab:megaimagenet}

\small
\centering
%\begin{minipage}[t]{0.45\textwidth}
\centering
\begin{tabular}{lccc}
\toprule
Model & Top 1  & Top 5 &  $\#$ of Parameters \\
\midrule
MEGA-ablate & 66.97 & 88.17 & 5.91M \\
EMA & 69.73 & 89.76 & 6.21M \\
\textsc{2D-SSM-r} & \textbf{70.11} & \textbf{90.19} & 5.96M\\
% \textsc{ 2D-SSM-c} & TBD & TBD & TBD\\
\bottomrule
\end{tabular}
      \centering
  \includegraphics[height = .4\linewidth,width = .856\linewidth]{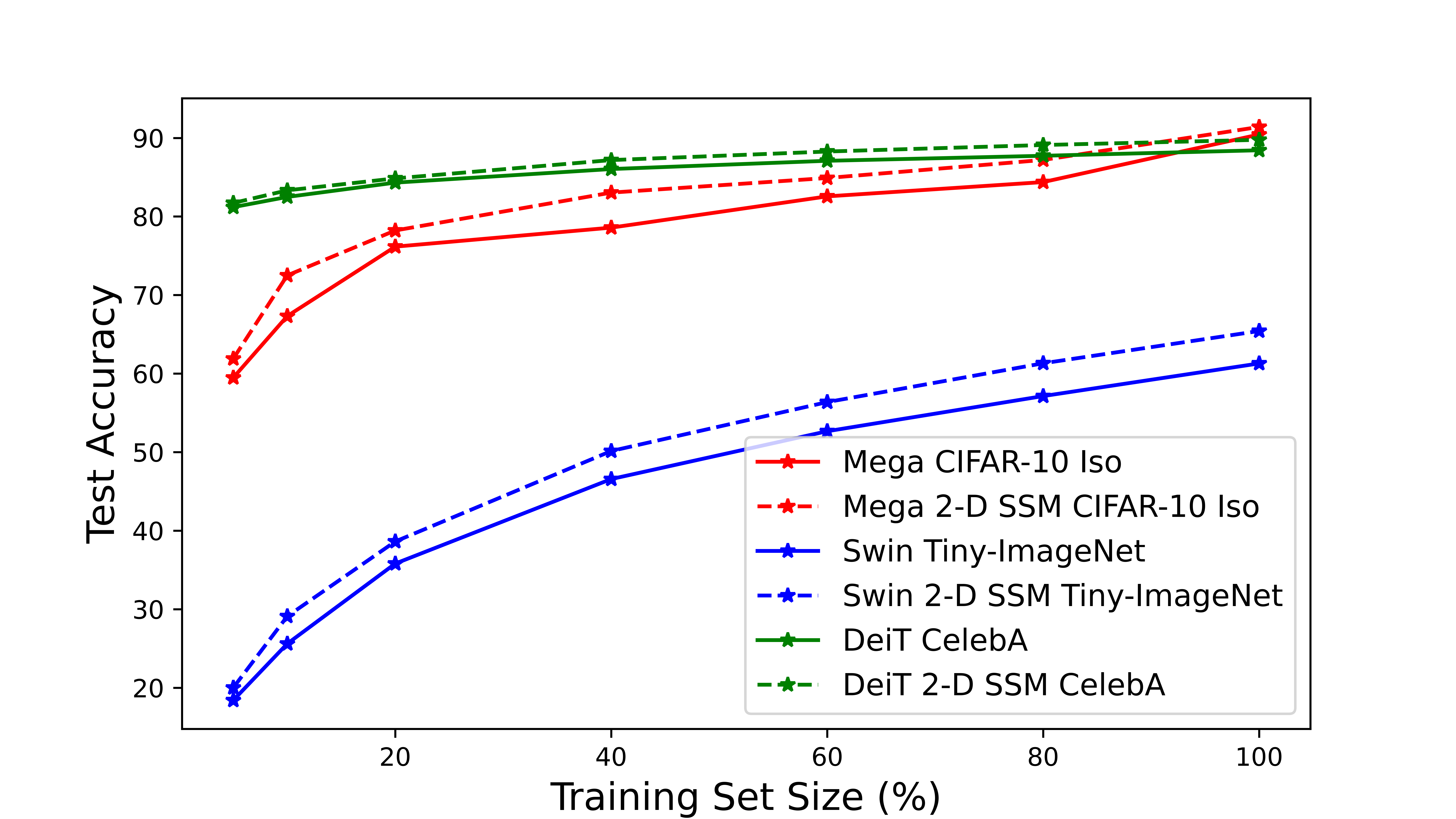}
  \caption*{Figure 3: The effect of the training set size.}
\refstepcounter{figure}\label{fig:x cubed graph}  \end{minipage}
\end{figure}

 \noindent{\bf Mega\enspace} { In Mega, we replace the EMA mechanism with 2-D SSM, which means that we only perform our layer on $Q,K$. We compare original Mega (with EMA) vs Mega-ablate (without EMA) and Mega 2-D SSM. We examined our model on CIFAR-10 Grayscale in an isotropic manner (without decreasing the image size along the architecture, and without patches}), which is part of the Long Range Arena benchmark \cite{tay2020long}. As shown in Tab.~\ref{tab:lra}, we improved the result by almost 1\% over MEGA, obtaining state-of-the-art results. 
We also conduct an experiment on the ImageNet-1K dataset \cite{deng2009imagenet}, and as shown in Tab. \ref{tab:megaimagenet}, we improve over MEGA's ViT-T results by $\sim 0.4\%$ in both Top 1 accuracy and Top 5 accuracy.  Finally, we check other small datasets (Tab. \ref{tab:Small Dataset}) and find superior results for combining Mega with 2-D SSM over baseline or other methods.

\noindent{\bf Comparisons against S4ND layer\enspace} S4ND is the only N-Dimensional SSM-based layer known to us; we compare our results against it by substituting SSM in ViT, Mega, and Swin. We conducted experiments on CIFAR100 and Tiny Imagenet. As indicated in Tab.~\ref{tab:Small Dataset}, S4ND performs very poorly when integrated into the original ViT backbone on CIFAR-100 (lower than the baseline by $1.21\%$) and sub-optimally when integrated into Swin and MEGA (achieves $~1\%$ or more lower accuracy on CIFAR-100 and Tiny-ImageNet for both backbones).

% \begin{wrapfigure}{r}{0.3\textwidth}
%   % \vspace{-4cm}
%   \centering
%   \includegraphics[scale=0.3]{sample-complexity-dpi-1200-nips-final.png}
%   \caption{The effect of 2-D SSM in different training set sizes on test accuracy, tested with various backbones on different datasets.}
%   \label{fig:x cubed graph}
% \end{wrapfigure}

\noindent{\bf Sample Complexity\enspace} We examine the behavior of our model with different backbones on different datasets over the baseline. As can be seen in Fig.~\ref{fig:x cubed graph}, 2-D SSM maintains improved results over the baseline for all backbones, which shows the data-efficient quality of our model.

%We follow MEGA's approach to data processing and its training regime to the letter. This follows DeiT's~\cite{touvron2021training} application of a long list of data augmentation and regularization methods, including Cutmix~\citep{yun2019cutmix}, Mixup~\citep{zhang2017mixup}, stochastic depth~\citep{huang2016deep}, repeated augmentation~\citep{hoffer2020augment}, Rand-Augment~\citep{cubuk2020randaugment}, and random erasing~\citep{zhong2020random}. }

% For Celeb-A, the original image size is 178x218, it is resized to 224x224 to match DEIT~\cite{touvron2021training} backbone and patch size. The dataset includes a 40-way multi-label attribute classification. We are reporting an average accuracy of all 40 tasks. We use the same data augmentation, and hyperparameters as DEIT, and train the models for 20 epochs, similar to S4ND~\cite{nguyen2022s4nd}. 
% \subsection{Ablations Studies}
%\subsection{Exploring variants}
\label{paragraph:runwithoutPE}

{\bf Removing the positional encoding (PE)\quad} We compare the empirical results obtained with real vs. complex kernels, with and without PE in Tab.~\ref{tab:PEcomplex}. Evidently, complex-SSM can be superior or inferior to real-based SSM, depending on the scenario. Additionally, our findings indicate that complex-SSM has a tendency to exhibit instability during training, which can result in poor performance. Stabilizing these models is an important direction for future research.

Running ViT backbones without PE decreases performance dramatically. In contrast, when our 2D-SSM layer is inserted into these backbones, they benefit from PE, and even without PE they outperform the original backbones with PE. {These findings support a innovative approach to introducing positional bias in ViT: rather than encoding positional information directly into the representation, it can be integrated into the computation by incorporating positional-dependent operators.}% This alternative approach offers a promising avenue for further exploration.}

% \begin{table}
% \caption{Ablations. For each model and dataset, we examine the effect of using original positional encoding and complex vs. real SSM.\label{tab:PEcomplex}}
% \centering
% \small
% \begin{tabular}{@{}l@{~}c@{~}cc@{~}cc@{~}cc@{~}cc@{}}
% \toprule
% \hfill Dataset: & \multicolumn{2}{c}{Tiny-ImageNet (Swin)} & \multicolumn{2}{c}{CIFAR100 (ViT)} & \multicolumn{2}{c}{CelebA (DeiT-T)} & \multicolumn{2}{c}{CIFAR-10 (Mega-ISO)}\\
% \cmidrule(r){2-3}
% \cmidrule(lr){4-5}
% \cmidrule(lr){6-7}
% \cmidrule(l){8-9}
% Model & with PE & w/o PE & with PE & w/o PE & with PE & w/o PE & with PE & w/o PE \\
% \midrule
% Baseline  &  61.29 & 58.97 & 73.26 & 64.09&  88.43 &  87.99 & 90.44 & 75.21 \\
% +Real 2D-SSM   &  65.77 & 65.44& 74.07 & 74.89  &  89.76 & 89.63&  91.31 & 90.68\\
% +Complex 2D-SSM &  3.28 & 2.16 & 73.91 & 74.67   & 89.84 &  89.83  &  90.46 & 90.79 \\
% \bottomrule
% \end{tabular}
% \end{table}

\begin{table}[h]
\caption{Ablations. For each model and dataset, we examine the effect of using original positional encoding and complex (C) vs. real (R) SSM. {The column $\delta$ represents the average difference for models with and without PE. As can be seen, our models are much more resistant to PE removal. %Bold denotes the model that is empirically best for the benchmark.%
\label{tab:PEcomplex}}}
\centering
\small
\begin{tabular}{@{}l@{~}c@{~}cc@{~}cc@{~}cc@{~}cc@{~}c@{}}
\toprule
\hfill Dataset: & \multicolumn{2}{c}{{Tiny-INet} (Swin)} & \multicolumn{2}{c}{CIFAR100 (ViT)} & \multicolumn{2}{c}{CelebA (DeiT-T)} & \multicolumn{2}{c}{CIFAR10 (Mega-ISO)}  & \multicolumn{1}{c}{Avg.}\\
\cmidrule(r){2-3}
\cmidrule(lr){4-5}
\cmidrule(lr){6-7}
\cmidrule(l){8-9}
\cmidrule(l){10-10}
Model & with PE & w/o PE & with PE & w/o PE & with PE & w/o PE & with PE & w/o PE  &  $\delta$ \\
\midrule
Baseline  &  61.29 & 58.97 & 73.26 & 64.09&  88.43 &  87.99 & 90.44 & 75.21 & -6.79  \\
\midrule
+Ours (R)   &  \textbf{65.77} & 65.44& 74.07 & \textbf{74.89}  &  89.76 & 89.63&  \textbf{91.31} & 90.68 & -0.07\\
+Ours (C) &  3.28 & 2.16 & 73.91 & 74.67   & \textbf{89.84} &  89.83  &  90.46 & 90.79 & \textbf{-0.01} \\
\bottomrule
\end{tabular}
\end{table}

\section{Limitations}
Despite the promising results presented in this paper, there are several limitations that should be considered. First, the current implementation of our proposed layer has relatively slow training times, as shown by the wall-clock measurements presented in the Experiments section \ref{sec:experiments}. This slow training time may be even more pronounced when applying our method to a longer two-dimensional sequence of patches, which could limit its applicability to tasks that require handling multi-dimensional long-range dependencies. One possible approach to mitigating this challenge is to use multi-dimensional parallel scanners, which could potentially reduce the training time of our layer. The main idea is to extend the work of S5 \cite{smith2022simplified}, which leverages 1-D parallel scanners to apply SSM on 1-D sequences to multi-dimensional parallel scanners and multi-dimensional sequences.

\section{Conclusions}
We present a novel spatial SSM-based layer that is more general than existing ones, encodes positional information by design, and is able to model spatial relations more expressively than other SSMs, including S4ND. When added to various ViT backbones, it is able to improve classification results on the various benchmarks without optimizing any aspect or other parts of the architecture. In future work, we would like to study the behavior of the layer in the context of spatial vision tasks, such as video processing, image segmentation, phrase grounding, and image inpainting. In the last task, the recursive view of the layer could be applied directly to impute missing pixels efficiently.

\section{Acknowledgments}
This work was supported by a grant from the Tel Aviv University Center for AI and Data Science (TAD), and the Blavatnik Family Foundation. The contribution of IZ is part of a Ph.D. thesis research conducted at Tel Aviv University.

%\section*{References}
\bibliographystyle{plain}
\bibliography{2dssm}
%\medskip

%\bibliographystyle{plainnat}
%bibliography{2dssm}

%%%%%%%%%%%%%%%%%%%%%%%%%%%%%%%%%%%%%%%%%%%%%%%%%%%%%%%%%%%%

%%%%%%%%%%%%%%%%%%%%%%%%%%%%%%%%%%%%%%%%%%%%%%%%%%%%%%%%%%%%%%%%%%%%%%%%%%%%%%%
%%%%%%%%%%%%%%%%%%%%%%%%%%%%%%%%%%%%%%%%%%%%%%%%%%%%%%%%%%%%%%%%%%%%%%%%%%%%%%%
% APPENDIX
%%%%%%%%%%%%%%%%%%%%%%%%%%%%%%%%%%%%%%%%%%%%%%%%%%%%%%%%%%%%%%%%%%%%%%%%%%%%%%%
%%%%%%%%%%%%%%%%%%%%%%%%%%%%%%%%%%%%%%%%%%%%%%%%%%%%%%%%%%%%%%%%%%%%%%%%%%%%%%%
\newpage
\appendix
\onecolumn
\label{appendix:complexity}

\section{Computing the kernel}
We discuss $x^{h}_{i,j}$,$k^{h}_{i,j}$. The same calculations hold for $x^{v}_{i,j}$,$k^{v}_{i,j}$.

$k^h_{i,j}$ can be written as:
\begin{align}
\label{eq:simplePowers}
\forall{i,j} : k^{h}_{i,j} = \sum_{z}^{2 * L_{max}} c_{z} A_1^{z_1} A_2^{z_2} A_3^{z_3} A_4 ^{z_4} B_{z_5}
\end{align}
For brevity and since it is not material for the method, we limit our exposition of the different power combinations of the system matrices $A_1,A_2,A_3,A_4$ and the input matrices $B_1,B_2$. As noted above, for each $k^h_{i,j}$ there are at most $2L_{max}$ elements.

\noindent{\bf Pre-Processing\enspace} The problem of finding $c_z$ for each element in the summation is a generalization of Pascal's triangle. In order to calculate the kernel, we calculate all the coefficients and the powers of $A_1 ... A_4$ up to the size of $L_1 , L_2$ and cache them before the training process.

During training, we employ a matrix multiplication process to compute $k^h_{i,j}$ with the learned parameters $A_1,...,A_4,B_1,B_2$.

Thus, for each cell there are at most $2L_{max}$ elements, and for each element, we save a constant number of $\chi$ values (the values of $z$ and $c_z$). As a result, the size of the cached matrix is bounded by $\mathcal{O}(L_{tot}L_{max})$.

It should be noted that currently in our method we cache the coefficients as One Hot Encoding and not the coefficient itself, and thus in our specific implementation we need to multiply the time complexity and memory complexity by $L_{max}$.

\section{Time and Memory Complexity of Our Method}
\label{appendix:timeMemoryComplexity}
To understand the complexity of our method, we will outline each step of the computation, starting from calculating the cache, creating the kernel during training, and computing the output $Y$ afterward.

As before, for brevity, we will refer only to the $x^h_{i,j}$, $k^{h}_{i,j}$ matrices caching, but the same holds for the vertical matrices. For simplicity, we assume $H=1$ (number of channels).

\noindent{\bf Caching\enspace} As noted in Section \ref{subsection:Computation}, for each cell $k^h_{i,j}$ in the horizontal kernel there are at most $2L_{max}$ elements. Also as noted in Section \ref{subsection:Computation}, $z_1,z_2,z_3,z_4 <= 2L_{max}$. Thus,
for each element in Eq. \ref{eq:simplePowers} we save $z_1,z_2,z_3,z_4,z_5$ and $\alpha_z$ values. In total, for each cell, we save $\chi_1 L_{max}$ values, where $\chi_1$ is a small constant. We have $L_{tot}$ cells and thus the total coefficient cached tensor for calculating $K_h$ is sized 
\begin{equation}
\chi_1*L_{max}L_{tot}
\end{equation}
From now on we will denote the tensors of horizontal coefficients as $CACHE \in \mathbb{R}^{\chi_1 \times L_{tot} \times L_{max}}$. Notice that there is a $CACHE$ tensor for each parameter, meaning $CACHE^h_{A_{1}},CACHE^h_{A_{2}},CACHE^h_{A_{3}},CACHE^h_{A_{4}},CACHE^h_{B}$.

\noindent{\bf Creating the Kernel\enspace} 
For brevity, we use real-valued diagonal $A_i \in [0,1]^{N}$ (after the sigmoid).  
First, we calculate the Vandermonde Matrix for each $A_1,A_2,A_3,A_4$ eigenvalues up to the highest power that exists in the kernel, which is $2L_{max}$, and denote $VAN_i = Vandermonde(A_i) \in \mathbb{R}^{2L_{max} \times N}$. 

Again, we have $L_{tot}$ cells in the horizontal kernel, each cell having $2*L_{max}$ elements. We take $CACHE_{A_{i}}$ which holds for each element its $A_i$ power, $z_i$ and creates a matrix that holds for each element its corresponding $A^{z_i}_{i}$ value.
\begin{equation}
\label{eq:complexity1}
O_{A_{i}}^h =  VAN_i[CACHE^h_{A_{i}}] \in\mathbb{R}^{{L_{tot} \times 2L_{max} \times N}}
\end{equation}

Now we multiply the matrices element-wise to obtain the final value of each element:
\begin{equation}
\label{eq:complexity2}
O^h_{pre-addition} =  O_{A_{1}}^h \odot O_{A_{2}}^h \odot O_{A_{3}}^h \odot O_{A_{4}}^h \odot O_{B}^h \odot O_{\alpha}^h 
 \in\mathbb{R}^{{L_{tot} \times L_{2L_{max}} \times N}}
\end{equation}

where $\odot$ denotes element-wise multiplication.
Now for each cell, we sum all the elements in the summation $k_{i,j}$, meaning summing over the second dimension:
\begin{equation}
O^h_{post-addition} =  sum(O^h_{pre-addition},d=1) \in\mathbb{R}^{{L_{tot} \times N}}
\end{equation}

Again, all the above steps are employed for the vertical axis as well, thus we are finally able to compute the kernel by using $C_1,C_2 \in \mathbb{R}^{N \times 1}$:
\begin{equation}
K = O^h_{post-addition} C_1 + O^v_{post-addition} C_2 \in \mathbb{R}^{L_1 \times L_2}
\end{equation}

Remembering that we actually used $n_{ssm}$ channels, the kernel size is $K \in \mathbb{R}^{L_1 \times L_2 \times n_{ssm}} $. It should be noted here that the calculation of the kernel is not dependent on the batch size $B$. %We would also like to note that currently, the method becomes quite costly as $L_{tot} or L_{max}$ increases. Optimizing the kernel calculation process was a major part of the work and would be an interesting follow-up work direction.

\noindent{\bf Forward Pass\quad}
Let B denote the batch size. We would like to convert input signal $U \in \mathbb{R}^{B \times L_1 \times \ L_2 \times H}$ to output signal $Y \in \mathbb{R}^{B \times L_1 \times \ L_2 \times H}$.  After calculating the kernel, we convert $U,K$ to the frequency domain through FFT:
\begin{equation}
    U_f = FFT(U), K_f = FFT(K)
\end{equation}
\begin{equation}
    Y = IFFT(U_f \odot K_f)
\end{equation}

This whole process costs us $O(BHL_{tot}\log(L_{tot}))$

Thus, our total forward pass time complexity is:
\begin{equation}
\mathcal{O}(\chi L_{TOT} L_{max} n_{ssm} N + BHL_{TOT}\log{L_{TOT}})
\end{equation}

\noindent{\bf Implementation detail\quad} We implemented the caching and matrix multiplication process with One Hot Encoding Vector of the powers and not by using floats representing the powers themselves. Thus, the size of each $COEFF^h_i$ in our implementation is multiplied by $L_{max}$, as is the time complexity of Eq. \ref{eq:complexity1} and \ref{eq:complexity2}.

% You can have as much text here as you want. The main body must be at most $8$ pages long.
% For the final version, one more page can be added.
% If you want, you can use an appendix like this one, even using the one-column format.
%%%%%%%%%%%%%%%%%%%%%%%%%%%%%%%%%%%%%%%%%%%%%%%%%%%%%%%%%%%%%%%%%%%%%%%%%%%%%%%
%%%%%%%%%%%%%%%%%%%%%%%%%%%%%%%%%%%%%%%%%%%%%%%%%%%%%%%%%%%%%%%%%%%%%%%%%%%%%%%

\section{Expressiveness}
\label{theorem:proof-expresinvess-2dssm}
\begin{theorem}
\label{theorem:2dssm-full-kernals}
  One channel of $2$-D SSM can express full-rank kernels 
\end{theorem}
\begin{proof} 
We start by restricting the system, output and input matrices:
\begin{align}
\label{eq:restriction1}
    A_1 = A_2 = A_3 = 1, \quad A_4 = 0, \quad C_1 = 1 , C_2 = 0, \quad  B_1 = 1 , B_2 = 0
\end{align}
For simplicity we assume that the values of the initial states are $0$: 
\begin{align}
\label{eq:restriction2}
    \forall i : i = -1 \rightarrow x_{i,j} = 0, \quad \forall j : j = -1 \rightarrow x_{i,j} = 0 
\end{align}
It suffices to show that (i) $K$ is a triangular matrix, and (ii) the diagonal of $K$ contains non-zero elements.
First, by plugging \ref{eq:restriction1},\ref{eq:restriction2} into the recurrent rule \ref{eq:reqRule}, it can be simplified:
\begin{align}
\label{reccurentRestrict}
   y_{i,j} = {x^h}_{i,j}, \quad {x^h}_{i,j} = {x^h}_{i,j-1} + {x^v}_{i,j-1} + u_{i,j}, \quad {x^v}_{i,j} = {x^h}_{i-1,j}
\end{align}
Given this simplification \ref{reccurentRestrict}, both (i) and (ii) can be easily proven by induction on the diagonals of $K$.
% Given this simplification \ref{reccurentRestrict}, both (i) and (ii) are are easy to prove, and the details are presented in the next two lemmas \ref{lem:usefullemma1},\ref{lem:usefullemma2}:

% \begin{lemma}
% \label{lem:usefullemma2}
%  $\forall i,j : 0 \leq i < j \leq L_2  \rightarrow K_{i,j} = 0$. 
% \end{lemma}

% \begin{lemma}
% \label{lem:usefullemma1}
%  $\forall i :  0 \leq i \leq L_{\textbf{min}}  : K_{i,i} = 1$. 
% \end{lemma}

To provide more insight into the proof, Eq. \ref{kernelRestrication} illustrates the values of $K$.
\begin{equation}
\label{kernelRestrication}
  \begin{bmatrix}
    1 & 1 & 1 & 1 & 1   \\
    0 & 1  & 2  & 3 & 4  \\
     \vdots & \ddots & 1 & 3 & 6 \\
     \vdots & \ddots &  \ddots & 1  & 4 \\
      0 &  \hdots & \hdots & 0  & 1   \\
  \end{bmatrix}
\end{equation}

And in general, since its clear from \ref{reccurentRestrict} that
\begin{align}
    y_{i,j} = {x^h}_{i,j} = {y}_{i,j-1} + {y}_{i-1,j-1} + u_{i,j}, \quad \forall j \rightarrow k_{0,j} = 1 
\end{align}

It easy to understand that the upper triangular of  $K$ can obtained from a rotation of Pascal's triangle.

\end{proof}

\section{Model Extension}
\label{sec:modelExtension}
\subsection{Bidirectional SSM} Using the State Space model, when calculating $x^h_{i,j}$ one only considers ${x_{
\hat{i},\hat{j}}}$ where {$\hat{i}\leq i$, $\hat{j}\leq j$}. To benefit from bidirectionality, we employ a version where we transpose the kernel in two or four directions (shown in the ablation experiments) and then sum the results. A similar mechanism for the 1-D case is used in S4, MEGA, and elsewhere.

\subsection{Multi-Axis Multidimensional SSM}

To enrich the dependencies that can be captured by our kernels, we use a weighted linear combination of kernels. Under the assumption that the system matrices are diagonal, the $N$ coordinates of the states are not dependent on each other, and can thus be calculated independently. Specifically, Eq. \ref{eq:simpleConvH} can be re-written separately per coordinate : $\forall g \in [N]  $
\begin{align}
\label{eq:kernelPerCoordinate}
    {x^h}_{i,j}[g] &= \sum_{0 \leq \hat{i} \leq i}\sum_{ 0 \leq \hat{j} \leq j} {k^h}_{\bar{i},\hat{j}}[g] u_{\hat{i},\hat{j}} 
\end{align}
Therefore, increasing $N$ will increase the number of kernels that make up $K$. By using this structure, the kernel can capture a variety of dependencies, where each coordinate focuses on a different type of dependency. Furthermore, this extension adds relatively negligible runtime when working with large batches, since the batch dimension does not affect the complexity when the kernel is computed.  Therefore, increasing $N$ will increase the number of kernels that compose $K$. 

\section{Justify Design Choices}

\subsection{Our Complex 2D-SSM}
\label{subsec:complexModelDeatils}
As explained in Sec. \ref{subsection:complexSSM}, our models employ real rather than complex SSMs. For reproducibility, here we provide a
detailed description of our complex SSM variant:
%\noindent{\bf Our 2D-complex SSM realization\quad}
The complex variant of our 2-D SSM model, still assumes  $\forall{t} , u_{i,j} \in \mathbb{R}^{1}$ , $\forall{t} , y_{i,j} \in \mathbb{R}^{1}$ and employs:
\begin{equation}
A_1 , A_2, A_3 ,A_4 \in \mathbb{C}^{N x N}, B_1,B_2 \in \mathbb{C}^{N x 1}, C_1,C_2 \in \mathbb{C}^{1 x N} 
\end{equation}
(diagonal matrices as above), and therefore
\begin{equation}
x^h_{i,j} \in \mathbb{C}^{N} , x^v_{i,j} \in \mathbb{C}^{N} 
\end{equation}

The output remains a real number $y_{i,j} \in \mathbb{R}^{1}$, and thus the real part of Eq. \ref{eq:reqRule} is used, namely $y^{out}_{i,j} = \operatorname{Re}(y_{i,j})$

For complex SSM, we save $\hat{A}_{i_{angle}} , \hat{A}_{i_{radius}} \in \mathbb{R}^{N \times N}$, and we calculate  $A_i \in \mathbb{C} ^{N \times N}$ in the following manner: 
\begin{equation}
A_{i_{angle}} = 2\pi sigmoid (\hat{A}_{i_{angle}}) ,\quad
% \end{equation}
% \begin{equation}
A_{i_{radius}} = sigmoid(\hat{A}_{i_{radius}})
\end{equation}
\begin{equation}
A_i = A_{i_{radius}} *(cos(A_{i_{angle}}) + i*sin(A_{i_{angle}})) \in \mathbb{C}^{N x N}
\end{equation}

The same goes for $B_1, B_2$. As for $C_1, C_2$, we perform the same operation without limiting the radius size (not applying a sigmoid to $\hat{C}_{i_{radius}}$). 

\subsection{No Weight Decay on the SSM core} While vision transformers and MEGA regularize the model via weight decay, SSM-based layers typically do not apply this \cite{gupta2022diagonal,gu2021s4}, since it drastically lowers the models' ability to learn, especially in the context of long-range dependencies. In general, higher values of the $A_i,B,C$ parameters do not seem to correspond with overfitting. Therefore, our method does not employ weight decay on those parameters.

\section{Experimental setup\label{sec:Experimentalsetup}}
We use PyTorch for all experiments. As a deliberate decision we choose to not perform hyper-parameter tuning of the backbone and training  procedure, apart from stochastic depth. All experiment results were averaged over seeds = $[0,1,2]$. For all datasets and backbones, we set $n_{ssm} = 8 , N=16 $ for all SSM-Based variants (SSM-2D real \& complex and S4ND).

\noindent{\bf Cifar-100 and Tiny imagenet\enspace}
For both datasets, we use as a baseline the experiments performed by \cite{lee2021vision}. 
This means we follow DeiT's~\cite{touvron2021training} application of a long list of data augmentation and regularization methods, including Cutmix~\citep{yun2019cutmix}, Mixup~\citep{zhang2017mixup}, stochastic depth~\citep{huang2016deep}, repeated augmentation~\citep{hoffer2020augment}, Rand-Augment~\citep{cubuk2020randaugment}, and random erasing~\citep{zhong2020random}. 
AdamW was used as the optimizer. Weight decay was set to 0.05 (apart from SSM layer where it was set to 0), batch size to 128, and warm-up to 10. All models were trained for 100 epochs, and cosine learning
rate decay was used. The initial learning rate was set to 0.003. In certain scenarios, we noticed that our models converge faster compared to the baseline approach. We discovered that a slight modification, specifically doubling the stochastic depth, proved to be instrumental in maximizing the model's performance.

When comparing S4ND, we used the same parameter scheme being used in 2-D SSM to perform a valid comparison, by making $C \in \mathbb{C}^{n_{ssm},N}$ instead of $C \in \mathbb{C}^{H,N}$ as in the original paper.

\noindent{\bf CelebA\enspace}
For Celeb-A, the original image size is 178x218, it is resized to 224x224 to match DeiT~\cite{touvron2021training} backbone and patch size. The dataset includes a 40-way multi-label attribute classification. We are reporting an average accuracy of all 40 tasks. We use the same data augmentation, and hyperparameters as DeiT, and train the models for 20 epochs, similar to the training procedure of S4ND~\cite{nguyen2022s4nd} on this datasets. 

\noindent{\bf Imagenet and CIFAR-10 Grayscale\enspace} We use the exact same training procedure including hyper-parameters, data augmentation and training environment as used in the git repository of the baseline~\cite{ma2022mega} for those datasets.  

\end{document}